\documentclass[pdflatex,sn-mathphys-num]{sn-jnl}

\usepackage{graphicx}%
\usepackage{multirow}%
\usepackage{amsmath,amssymb,amsfonts}%
\usepackage{amsthm}%
\usepackage{mathrsfs}%
\usepackage[title]{appendix}%
\usepackage{xcolor}%
\usepackage{textcomp}%
\usepackage{manyfoot}%
\usepackage{booktabs}%
\usepackage{algorithm}%
\usepackage{algorithmicx}%
\usepackage{algpseudocode}%
\usepackage{listings}%

\newtheorem{theorem}{Theorem}
\begin{document}

\title[Design and Control of Modular Soft-Rigid Hybrid Manipulators with Self-Contact]{Design and Control of Modular Soft-Rigid Hybrid Manipulators with Self-Contact}


\author*[1]{\fnm{Zach J.} \sur{Patterson}}\email{zpatt@mit.edu}
\equalcont{These authors contributed equally to this work.}

\author[1]{\fnm{Emily} \sur{Sologuren}}\email{esolo@mit.edu}
\equalcont{These authors contributed equally to this work.}

\author[2, 3]{\fnm{Cosimo} \sur{Della Santina}}\email{C.DellaSantina@tudelft.nl}

\author[1]{\fnm{Daniela} \sur{Rus}}\email{rus@csail.mit.edu}

\affil*[1]{\orgdiv{Computer Science and Artificial Intelligence Laboratory}, \orgname{MIT}, \orgaddress{\city{Cambridge}, \state{MA}, \country{USA}}}

\affil[2]{\orgdiv{Department of Cognitive Robotics}, \orgname{Delft University of Technology}, \orgaddress{\city{Delft}, \country{The Netherlands}}}

\affil[3]{\orgdiv{Institute of Robotics and Mechatronics}, \orgname{German Aerospace Center (DLR)}, \orgaddress{\city{Wessling}, \country{Germany}}}


\abstract{Soft robotics focuses on designing robots with highly deformable materials, allowing them to adapt and operate safely and reliably in unstructured and variable environments. While soft robots offer increased compliance over rigid body robots, their payloads are limited, and they consume significant energy when operating against gravity in terrestrial environments. To address the carrying capacity limitation, we introduce a novel class of soft-rigid hybrid robot manipulators (SRH) that incorporates both soft continuum modules and rigid joints in a serial configuration. The SRH manipulators can seamlessly transition between being compliant and delicate to rigid and strong, achieving this through dynamic shape modulation and employing self-contact among rigid components to effectively form solid structures. We discuss the design and fabrication of SRH robots, and present a class of novel control algorithms for SRH systems. We propose a configuration space PD+ shape controller and a Cartesian impedance controller, both of which are provably stable, endowing the soft robot with the necessary low-level capabilities. We validate the controllers on SRH hardware and demonstrate the robot performing several tasks. Our results highlight the potential for the soft-rigid hybrid paradigm to produce robots that are both physically safe and effective at task performance.}




\maketitle

\section{Introduction}\label{sec1}

Soft robots promise safe applications in human-centered environments \cite{abidiIntrinsicSafetySoft2017,Rus2015}; however, their potential utility in gravity-dominated settings (i.e. excluding aquatic or outer-space environments), is constrained by their load-bearing capacity. 
This observation aligns with the fact that very few terrestrial animals utilize entirely deformable or soft structures for manipulation \cite{kim_soft_2013}. Elephant trunks and giraffe tongues are notable examples often referenced in the robot literature \cite{mcmahan_robotic_2004}. More frequently, terrestrial (and many non-terrestrial) morphologies utilize rigid and soft tissues in tightly integrated arrangements \cite{rockwellComparativeMorphologyVertebrate1938,robertsFlexibleMechanismsDiverse2011,blobDiversityLimbBoneSafety2014,desatnik2023soft,donoghue2002origin}. 

Recently, researchers within the robotics community have increasingly been taking inspiration from these materially complex terrestrial morphologies and developing novel soft-rigid hybrid robots using techniques developed by the soft robotics community in the past decade \cite{bern_simulation_2022,obayashi_control_2023,culha_design_2017,patterson2023modeling}.

Taking inspiration from this integration of rigid and soft materials, we build upon concepts such as the 2D manipulator in Bern et al. \cite{bern_contact-rich_2022}. This cable-actuated foam manipulator was capable of modulating its stiffness on a continuum between "soft" and "rigid" states by compressing its rigid plates, utilizing both geometric deformation and self-contact between adjacent rigid plates to become stiffer. However, this manipulator concept could not support its own weight under gravity and was limited to a planar setting. Additionally, the frequency of self-contact makes the application of classical robot and soft robot control solutions difficult, and feedback control methods for such a system have not been demonstrated.

In this work, we present a new design for stiffness modulation to 3D. The manipulator consists of an arbitrary number of cable-actuated foam modules that can bend and compress, along with servo motors in between modules. The manipulator is capable of supporting itself under gravity and can also manipulate significant additional loads. The continuum modules provide a large and highly redundant workspace. See Figure 1
for an image of the robot, an overview of some of its main features, and an overview of our modeling and control methodology.

The concept of stiffness modulation in robotics is well studied - there are important examples in both the rigid-bodied robotics literature \cite{schiavi_vsa-ii_2008,wolf_variable_2016} and in the more recent soft robotics literature \cite{huh_design_2012,park_design_2014,stella_prescribing_2023,bruderIncreasingPayloadCapacity2023,baines2022multi}. A comprehensive review of all stiffness modulating robots is outside of the scope of this paper. The type of stiffness modulation present fits with the nomenclature characterized by Yang et al. \cite{yang_principles_2018} as a combination of antagonistic stiffness tuning and segment locking, both under the category of structure-based stiffness tuning (as opposed to methods that utilize particle jamming or material phase change). Because these distinct forms of stiffness modulation are present in our design results in distinct qualitative features, most notably that the bending stiffness is continuously controllable through coordinated compression using the cables, but only up to and until the plates make contact. At that point, the rigid plates act like a discrete, hard constraint and essentially imbues infinite stiffness in the direction of the bend.

Together with this new robotic concept, we present a novel control methodology that can equip this structure with the capability of moving precisely in space and exploiting its embodied intelligence. When plates in the foam modules are not making contact with each other, the modules can be appropriately controlled using existing recent results from the soft robot control literature, such as the soft PD+ controller \cite{della_santina_model_2021}. However, the introduction of self-contact by the plates creates non-smooth dynamics that are not readily addressed by the standard framework. While torque control of the motors can mitigate this effect to some degree, it still produces undesirable performance and, at worst, dangerous instability in the case of a large enough proportional feedback term. Therefore, we also present a novel control method for controlling systems experiencing self contact. The controller is essentially an extension of the soft PD+ controller \cite{santibanezGlobalAsymptoticStability1999, della_santina_model-based_2020} with an additional term to take the contact condition into account. We also present an impedance controller for regulating interactions with the environment. Finally, we demonstrate the robot on several tasks the require both flexibility and load bearing capability.

\noindent \begin{figure}
\centering
\includegraphics[width=0.95\textwidth]{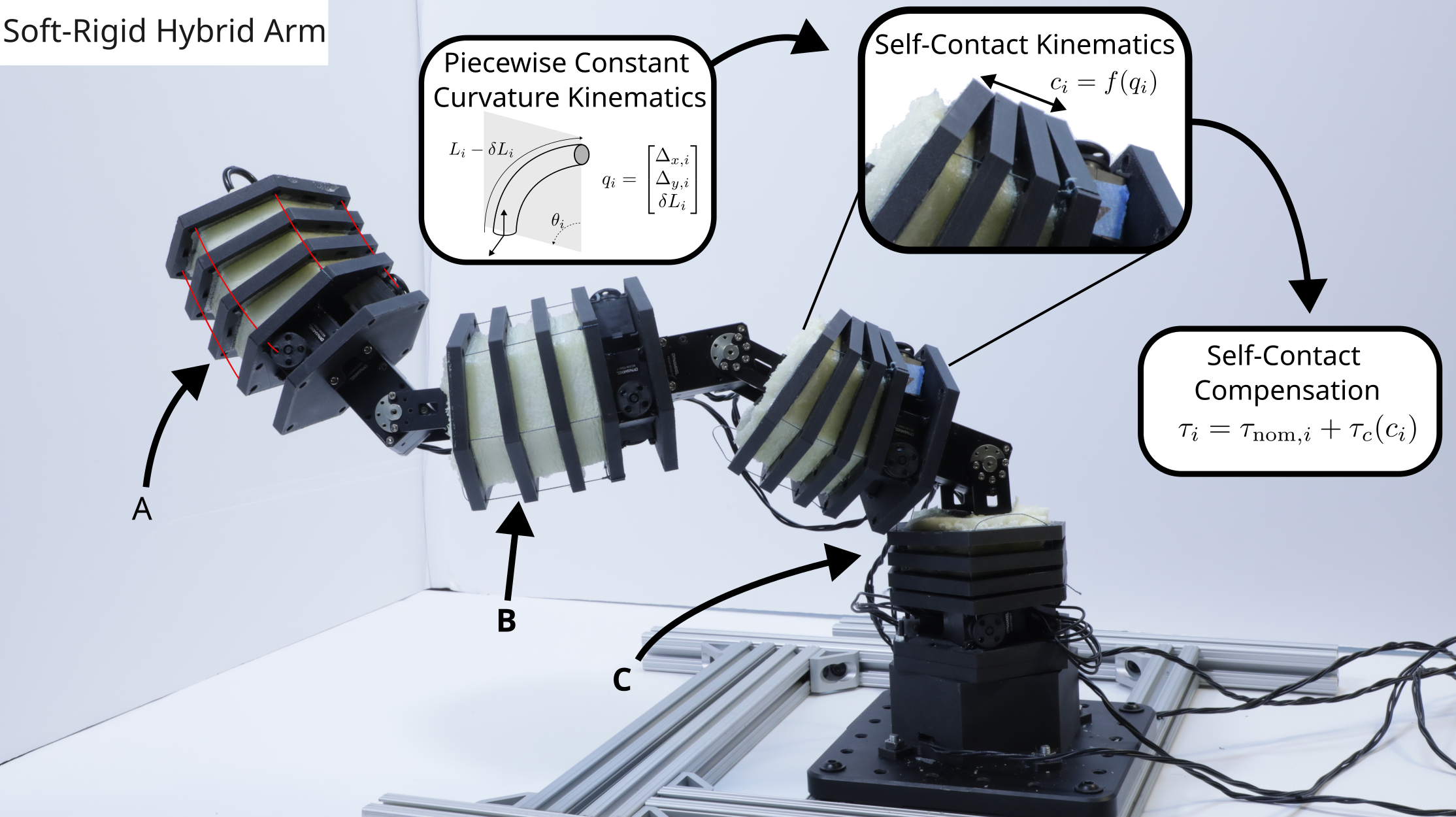}
\caption{A modular soft-rigid hybrid arm that can operate as a series of rigid bodies or soft segments. The figure shows a control workflow that incorporates contact compensation based on forward kinematics calculations. (a) Each module is controlled by a set of three tendons (visible ones on the first module are highlighted in red). (b) When a module is uncompressed, the module acts like a soft segment, particularly under load. (c) When the module is completely compressed, it acts as a de facto rigid body.}\label{fig:fig1}
\end{figure}

\section{Results}\label{sec2}

\subsection*{Soft-rigid hybrid modules}
 Our robot is modular in its design and is composed of both continuum modules and rigid joints. Figure 2
 shows an overview of the design of the modules and the robot. Each continuum module is composed of rigid 3D printed plates embedded in highly deformable rubber foam. The unactuated structure is free to bend, compress, shear, and twist. To actuate each module, we use three motor driven tendons. With this actuation methodology, we are able to actively bend and compress the modules. 

 \noindent \begin{figure}
\centering
\includegraphics[width=0.6\textwidth]{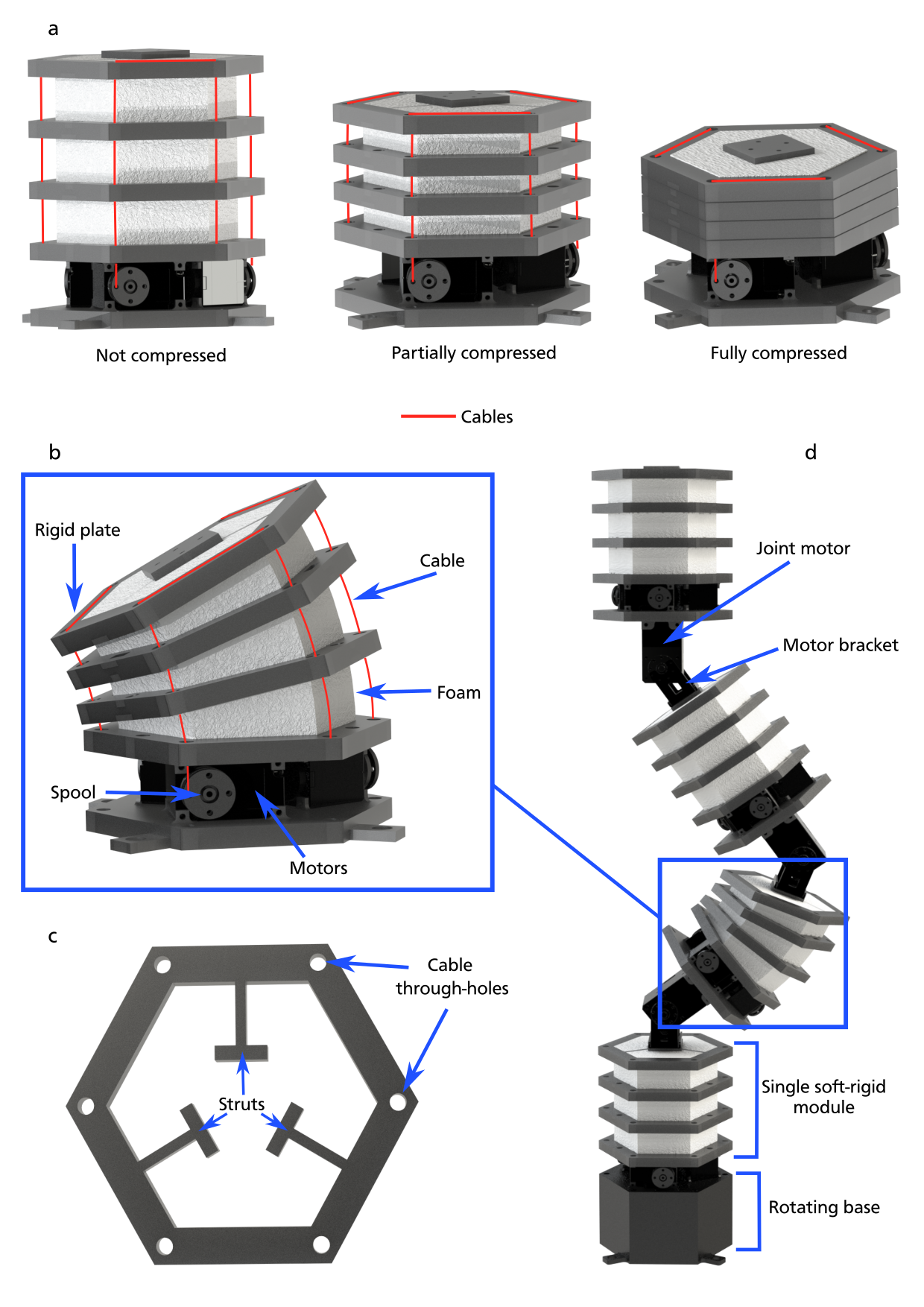}
\caption{\textbf{Design overview of the soft-rigid manipulator:} (a) Panel visualizing a soft-rigid module transitioning between its soft and rigid states. As the cables contract, the rigid plates are pushed closer together, increasing overall stiffness. (b) A labelled soft-rigid module that relies on three dynamixel motors, spools, and cables to achieve various levels of stiffness and shape. (c) A single 3D printed hexagon plate, featuring its struts for better foam adhesion and through-holes for cables. (d) A full body rendering of the soft-rigid arm, composed of four soft-rigid modules.}\label{fig:design}
\end{figure}

 The most interesting feature of these modules is that the effective stiffness of the structure is dependent on the shape. When the module is unactuated (Fig. 
 2a), it is free to deform and is quite compliant (we characterize the stiffness quantitatively in the sequel). On the other hand, when the module is fully compressed, it is essentially rigid. The module can also be continuously modulated between these two states. 
During active bending, the modules can bend until the rigid plates make contact (see Fig. 
2b for a bent module). The magnitude of the bend at which the plates make contact obviously decreases the more the module is compressed. 

To assemble a robot, several design choices could be made. One options is to stack the modules directly on top of each other, resulting in a traditional kinematic chain of continuum modules \cite{robinsonContinuumRobotsState1999}. Instead, we chose to stack modules with traditional rigid joints in between and with an additional joint at the base to rotate the robot about the vertical axis (see Fig. 
2d). The height of the entire manipulator is 0.68m. 

\subsection*{Robot stiffness variation and workspace}

To characterize the bending stiffness of a soft-rigid module, we perform a simple suspended weight experiment wherein we suspend a weight from a horizontal module and measure the displacement between the deformed and undeformed module. We make the assumption that the module undergoes pure bending. While this neglects what are likely significant shear deformations, we are primarily interested in the change in stiffness of a module as it is compressed. We also utilize a standard linear beam assumption to approximate the stiffness, resulting in a stiffness calculation of $k = mg/\Delta h$, where $k$ is the bending stiffness, $m$ is the mass, $g$ is gravity, and $\Delta h$ is the vertical displacement. The results of our experiments and an inset depicting the methodology are shown in Fig. 
3a. There are several interesting characteristics to point out here. First, as the module compresses but before it is fully compressed, we see a modest change; the stiffness approximately doubles. However, when the module fully compresses, the stiffness increases by approximately 30 times the baseline. This is due to the discrete contact between the plates.

\noindent \begin{figure}
\centering
\includegraphics[width=0.75\textwidth]{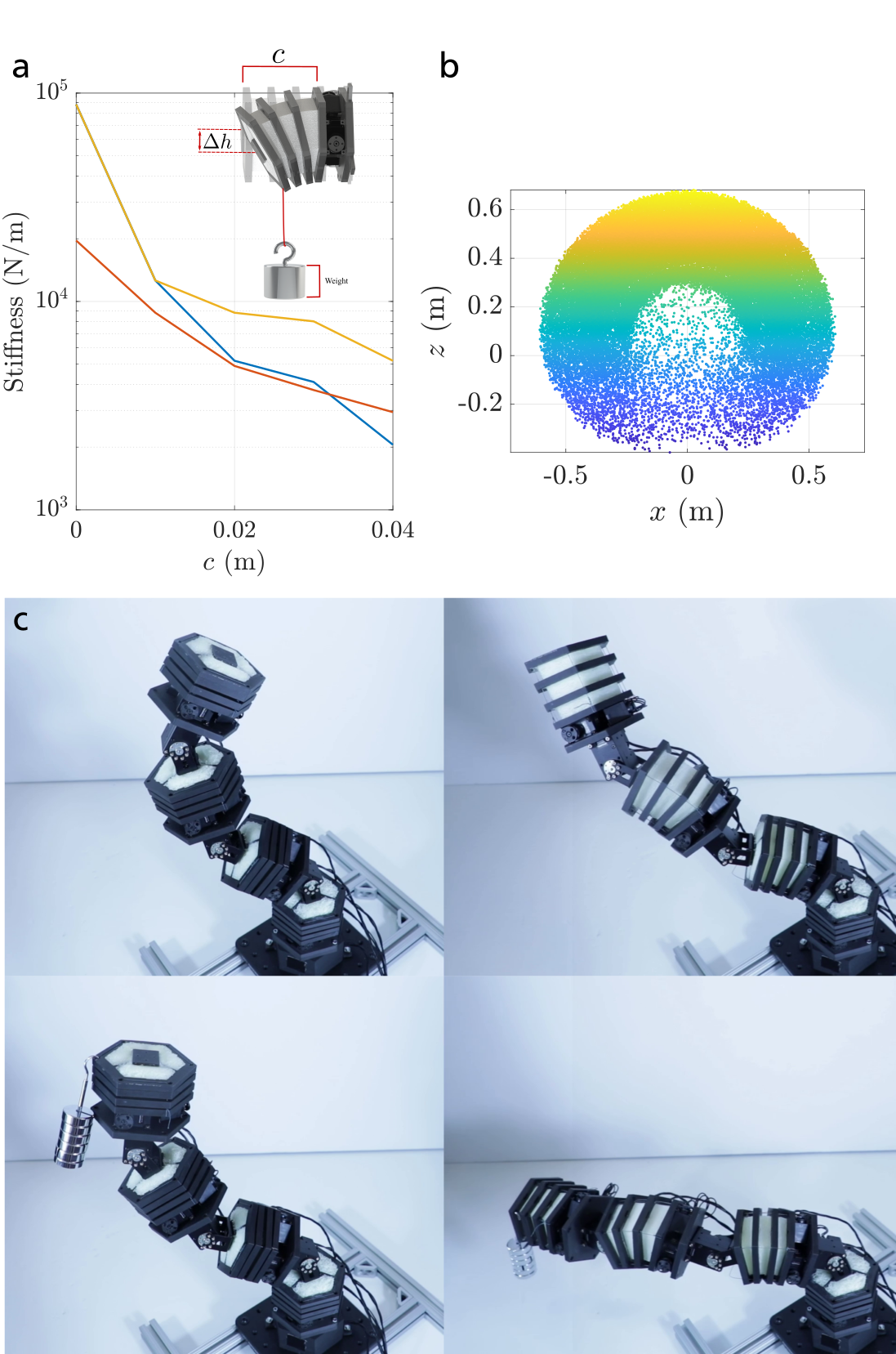}
\caption{\textbf{Characterization of Soft-Rigid Module:} (a) Plot of the bending stiffness test results. Inset is a rendering of the test performed on a module of the soft-rigid arm. (b) Planar slice of the manipulator workspace (which is radially symmetric). (c) Demonstration of manipulator supporting a 300 gram mass while the modules are in a rigid state (left) and a flexible state (right). }\label{fig:characterization}
\end{figure}

We next characterize the workspace. To do this, we use the forward kinematics of the end effector of the robot to generate a point cloud of reachable points by randomly varying the configuration of the robot. Figure 
3b shows a slice of the point cloud. We note that much of the workspace is highly redundant because of the nature of continuum robots \cite{Jones2006}. We also observe that the center of the workspace, while reachable, is noticeably more sparse, implying that there are far less configurations that reach this area. This also makes intuitive sense; to reach these points the robot must be near fully compressed, greatly reducing the redundancy of the continuum modules. 

Finally, to give a qualitative demonstration of the performance of the entire robot under load for different configurations, we performed a simple experiment in which we suspended 300 grams from the robot when it was fully compressed (maximum stiffness) and fully elongated (maximum flexibility). The state of the joint motors was set to be identical in both cases. We observe the result of the experiment in Fig. 
3c. On the left, we see that the rigid robot moves slightly from the unloaded point. On the right, we see that the weight causes the robot to deform until it hits the ground.

\subsection*{Modeling and feedback control}
Because the rigid plates constrain our modules from undergoing very large displacements, we use the Piecewise Constant Curvature (PCC) kinematic assumption to parametrize the modules' states \cite{Jones2006}. We utilize the formulation discussed in \cite{della_santina_improved_2020}, which is singularity free. Thus, for each module, the state is ${q}_{\mathrm{mod},i} = [\Delta_{x,i}, \Delta_{y,i}, \delta L_{i}]^T$. Each rigid joint is modeled using a standard rotational joint, represented by $\theta_i$ \cite{murray_mathematical_1994}. We can then stack the variables to form our total state as ${q} = [\theta_1, {q}_{\mathrm{mod},1}, ..., \theta_n,{q}_{\mathrm{mod},n}]^T$, where $n$ is the total number of modules. The dynamics of the robot are then derived, following \cite{della_santina_improved_2020}, in the following form:
\begin{equation}\label{eq:dynamics}
    M({q}){\ddot q} + C({q},{\dot q}){\dot q} + G({q}) + K({q}) + D{\dot q} = A({q}){\tau} + J^T f_{\mathrm{ext}},
\end{equation}
where $M({q})$ is the inertia matrix, $C({q},{\dot q})$ is the Coriolis matrix, $G({q})$ is gravitational force, $K({q})$ is the elastic force, $D$ is the damping matrix, $A({q})$ is the input matrix, ${\tau}$ is the input vector, $J$ is the end effector Jacobian, and $f_{\mathrm{ext}}$ is a force on the end effector. To calculate the inertial, Coriolis, and gravitational terms online, we implemented Featherstone's recursive algorithms \cite{featherstone_rigid_2014}. The elastic force, $K({q})$ is assumed to be linear and thus takes the form $K{q}$, where $K$ is the diagonal Hessian of the elastic potential. In this case, the terms of the diagonal corresponding to rigid joints are zero and the others are calibrated based on system response. Damping terms in diagonal matrix $D$ are estimated based on the textbook system identification method as in e.g. \cite{sciavicco2012modelling}.
In order to drive the system (\ref{eq:dynamics}) to trajectories in the state $ q(t)$, we chose to utilize a PD+ controller adapted for soft robotics, similar to those described in \cite{della_santina_model_2021}. The controller takes the form 
\begin{equation}\label{eq:state_c}
    {\tau} = A^{-1}(M{\ddot q_d} + C{\dot q} + G({q}) + K{q_d} + D{\dot q_d} + K_{\mathrm P}({q_d} - {q}) + K_{\mathrm D}({\dot q_d} - {\dot q}) + {F}_c),
\end{equation}
where ${q_d}(t)$ is the desired trajectory, $K_{\mathrm P}$ is the proportional feedback gain matrix, $K_{\mathrm D}$ is the derivative feedback gain matrix, and ${F}_c$ is a contact compensation term that we will now describe. 
Our objective with ${F}_c$ is to prevent the controller from causing interpenetration of the rigid plates which can result in the destruction of the robot. This problem could be addressed at the planning level by preventing dangerous desired configurations by checking the inverse kinematics, but this incurs additional computational cost. We could also use optimization-based methods such as control barrier functions (CBFs) \cite{patterson2023safe} or model predictive control (MPC) \cite{koenemann2015whole}. These approaches have notable drawbacks that motivate another option. CBFs can result in an infeasible QP for a large number of constraints \cite{xiao2022sufficient}. Using MPC to solve the control problem for general nonlinear systems is an open question \cite{diehl2009efficient}. Here we take inspiration from Khatib \cite{khatib1986real} and write our self-contact compensation as a term in the low level torque controller. We could choose a variety of functions to accomplish the goal of preventing dangerous self-contact. One line of thinking is to gradually attenuate the terms that would cause the controller to push past the contact point, specifically $K{q_d}$ and $K_{\mathrm P}({q_d} - {q})$. To accomplish this, we design ${F}_{c,i}$ for a single module as
\begin{equation}
    {F}_{c,i} = \sigma(c_i({q}))(K{q_d} + K_{\mathrm P}{q_d} - K_{\mathrm P}{S}({q}),
\end{equation}
where $c_i({q})$ is the function for the distance between the plates derived from forward kinematics and $\sigma(c_i({q}))$ is a sigmoid function that saturates as $c_i \rightarrow 0$:
\begin{equation}
    c_i = 2(d\frac{L_0 + \delta L_i}{\sqrt{\Delta_{x,i}^2 + \Delta_{y,i}^2}} - d)\sin(\frac{\sqrt{\Delta_{x,i}^2 + \Delta_{y,i}^2}}{6d}),
\end{equation}
\begin{equation}
    \sigma(c) = -\frac{e^{-k_c c}}{e^{-k_c c} + 1}.
\end{equation}
$k_c$ determines the steepness of the sigmoid's slope about $c=0$. 
The term ${S}({q})$ saturates the feedback contribution to ${F}_{c,i}$
\begin{equation}
    {S}({q}) = [\phi(\Delta_{x,i}), \phi(\Delta_{y,i}), \phi(\delta L_{i})]^T,
\end{equation}
where
\begin{equation}
    \phi(\Delta_{x,i}) = 2\Delta_{x,\mathrm{max}} \frac{1-e^{-\Delta_{x,i}/\Delta_{x,\mathrm{max}}}}{1+e^{-\Delta_{x,i}/\Delta_{x,\mathrm{max}}}}
\end{equation}
and $\Delta_{x,\mathrm{max}}$ is a positive constant. ${F}_c$ can then be constructed as 
\begin{equation}
    {F}_c = [0, {F}_{c,1}, ..., 0, {F}_{c,n}]^T,
\end{equation}
where the null terms correspond to the motor states $\theta_i$. The closed loop system of (\ref{eq:dynamics}) and (\ref{eq:state_c}) can be shown to be uniformly bounded for trajectories ${q_d}(t)$, ${\dot q_d}(t)$ (see the SM). We emphasize that this controller is designed to reach ${q_d}$ when the trajectory will not cause interpenetration while converging to a different trajectory before interpenetration can occur. Thus, we should observe that the controller performs well in the both the tracking and the set point problem before $c( q) \rightarrow 0$.

We test performance of controller (\ref{eq:state_c}) in step response, disturbance rejection, and trajectory tracking tasks. Results for step response and disturbance rejection are shown in Fig. 
4a-c. We observe that most of the states successfully converge to their set points with the exception of $\delta L_1$, for which the set point would have caused interpenetration. Fig. 
4d-f show the results for the trajectory tracking task. We observe that tracking for the joint motors is near perfect (Fig. 
4E), whereas tracking for the module states can sometimes lag the desired state (Fig. 
4d), likely due to model mismatch and small gains. For videos of the disturbance rejection and trajectory tracking, see Videos V1 and V2 respectively.

\noindent \begin{figure}
\centering
\includegraphics[width=0.99\textwidth]{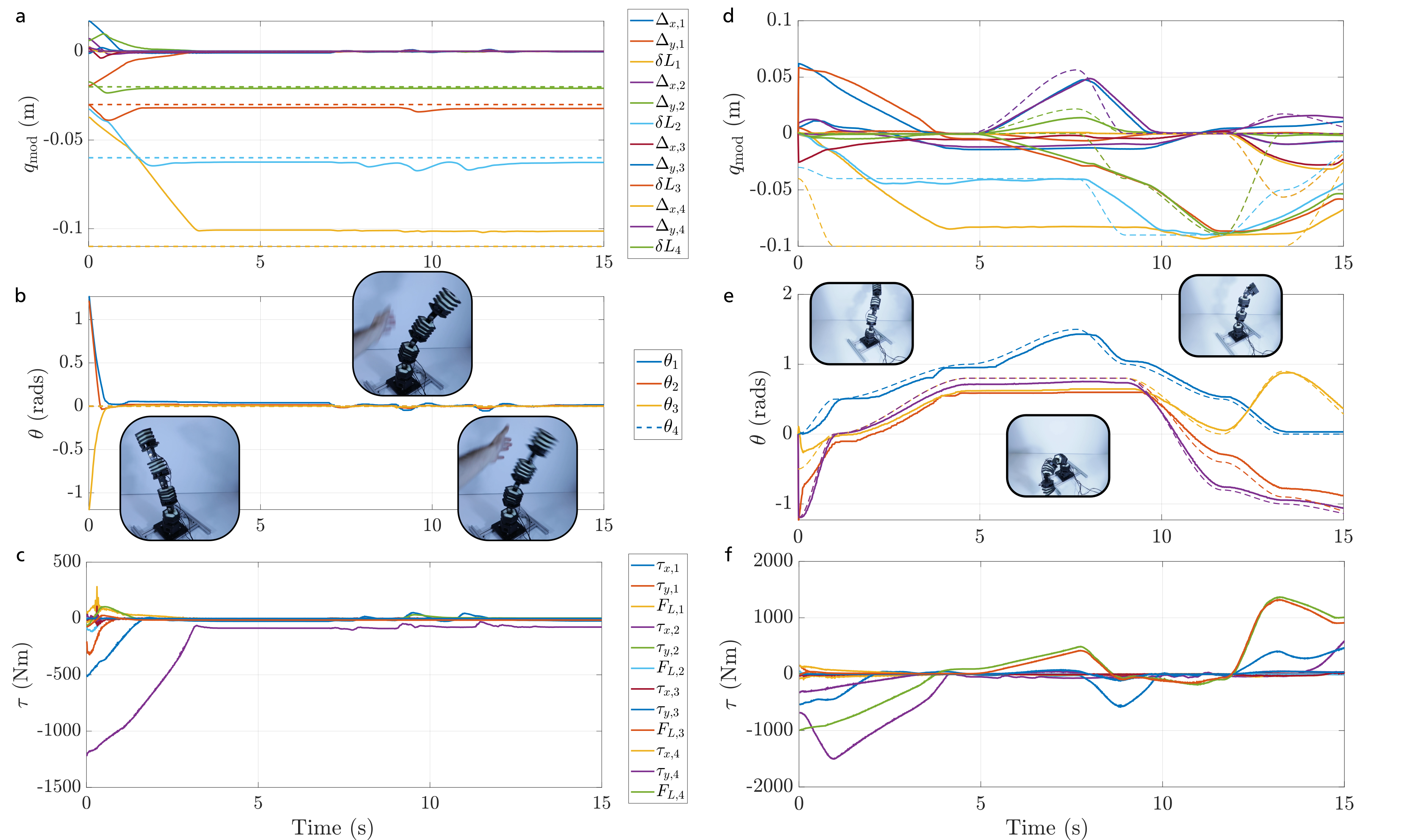}
\caption{\textbf{Disturbance Rejection and Trajectory Tracking:} Step response and disturbance rejection for (a) modules and (b) joints. The inset images are snapshots of the robot as it is perturbed from the set point. (c) Torques output by the controller. Trajectory Tracking for (d) modules and (e) joints. The inset images are snapshots during the trajectory. (f) Torques output by the controller.}\label{fig:step}
\end{figure}

We also explore controlling the manipulator using impedance control \cite{lynch2017modern}. Depending on the application, such controllers may be more desirable than configuration space control due to the fact configuration space control necessarily stiffens the system with respect to interactions \cite{della_santina_model-based_2020}, whereas impedance control simulates a mass spring damper system which allows the user to choose the rendered stiffness. Such an approach therefore complements the flexible nature of the manipulators. We specify the following impedance controller, taking inspiration from various sources \cite{della_santina_model-based_2020,khatibUnifiedApproachMotion1987,ott2008cartesian}, as
\begin{equation}\label{eq:imp}
\begin{split}
    {\tau}_{\mathrm{I}} &= A^{-1}(\mu({x},{\dot x}){\dot x_d} + J^TJ^{+T}(K{q} + D{\dot q}) + J^T (\Lambda {\ddot x_d} + K_c({x_d} - {x}) + D_c({\dot{x}_d} - {\dot x}))\\ 
    &+ (I - J^T J^{+T})(K_{P}({q_d} - {q}) - K_{D}{\dot q}),
\end{split}
\end{equation}
where ${ \dot x} = J{\dot q}$ is the Cartesian velocity of the end effector, ${x_d}$ is the desired Cartesian position, $K_c$ is the desired Cartesian stiffness, $D_c$ is the desired Cartesian damping, $\mu({x},{\dot x})$ is the Cartesian coriolis matrix, and $\Lambda$ is the Cartesian inertia matrix \cite{khatibUnifiedApproachMotion1987} defined as 
\begin{equation}
    \Lambda = (JM^{-1}J^T)^{-1},
\end{equation}
and $J^{+T}$ is the dynamically consistent pseudo-inverse
\begin{equation}
    J^{+} = M^{-1}J^T\Lambda.
\end{equation}
Assuming stable nullspace dynamics (see \cite{ott2008cartesian} for a discussion on this topic), this controller is asymptotically stable about ${x}={x_d}+K_c^{-1}f_{\mathrm{ext}}$, ${\dot x} = 0$ in the presence of external force $f_{\mathrm{ext}}$.

We demonstrate the impedance controller in a set point task subject to disturbances. For a better idea of the performance of the controller, see Video V3. Results are shown in Fig. 5
. Fig. 
5a shows the Cartesian position of the system compared to the desired trajectory when the system is subject to repeated disturbances. In this case, the median $L^2$ norm of the errors is roughly 4 centimeters over the length of the trial (compared to the robot length of 68 centimeters). When not disturbed, the Cartesian error is approximately 5 millimeters (see Fig. S1). Figure 
5b shows the motor torques. Figures 
5c and 
5d show the evolutions of the module states and joint states respectively. Note that while the Cartesian position stays relatively close to our desired position, the configuration state of the system frequently undergoes large changes in response to the external disturbances. Figure 
5e shows still images of the experiment. 

\noindent \begin{figure}
\centering
\includegraphics[width=0.99\textwidth]{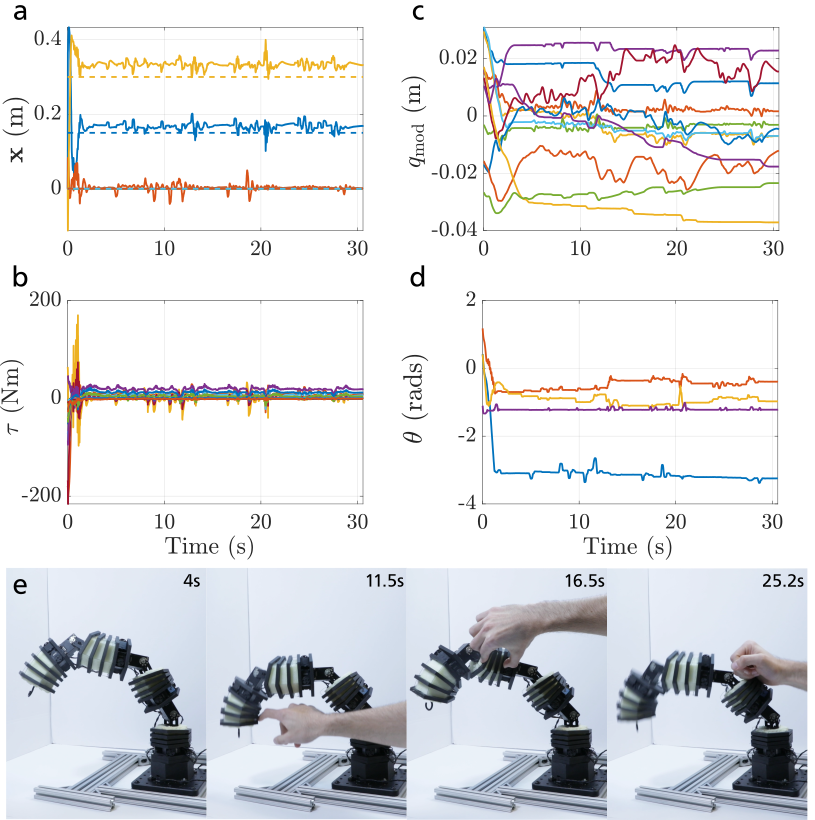}
\caption{\textbf{Impedance Control:} Depicts an experiment in which impedance controller (\ref{eq:imp}) renders a desired spring damper system at the end effector compared to desired trajectory. (a) Cartesian coordinate of the end effector. (b) Control inputs. (c) Module state trajectory. (d)  Joint motor state trajectory. (e) Snapshots from the trial.}\label{fig:impedance}
\end{figure}

\subsection*{Demonstrations}

To show the kinematic and load bearing functionality of the arm, we perform two demonstrations. For these experiments, the manipulator is outfitted with a hook to serve as a simple end-effector that can be used to pick up objects. For each experiment, the robot must maneuver its way through an "obstacle course" that requires it to utilize its continuum kinematics to reach the goal. An image of the obstacle course and robot is shown in Fig. 
6a. As this is a demonstration of the hardware functionality, we want to emphasize that these trajectories are manually specified offline, i.e. the robot is not autonomously generating its own trajectories. 

In the first task, we have the robot maneuver through several holes to grab a red mug. This task requires the robot to utilize the continuum bending capability to successfully reach through the holes. Snapshots from the experiment are shown in Fig. 
6. In the second task, we demonstrate the robotic arm's ability to recover cup of 200 grams of weights while retaining enough kinematic redundancy to pass through a section of the obstacle course. To pick up and maneuver the cup of 200 grams of lab weights, the manipulator is selectively stiffened to support the weight. Snapshots from the cup of weights experiment are shown in Fig. 
7. Videos of both demonstrations can be found in Videos V4 and V5.

\noindent \begin{figure}
\centering
\includegraphics[width=1.0\textwidth]{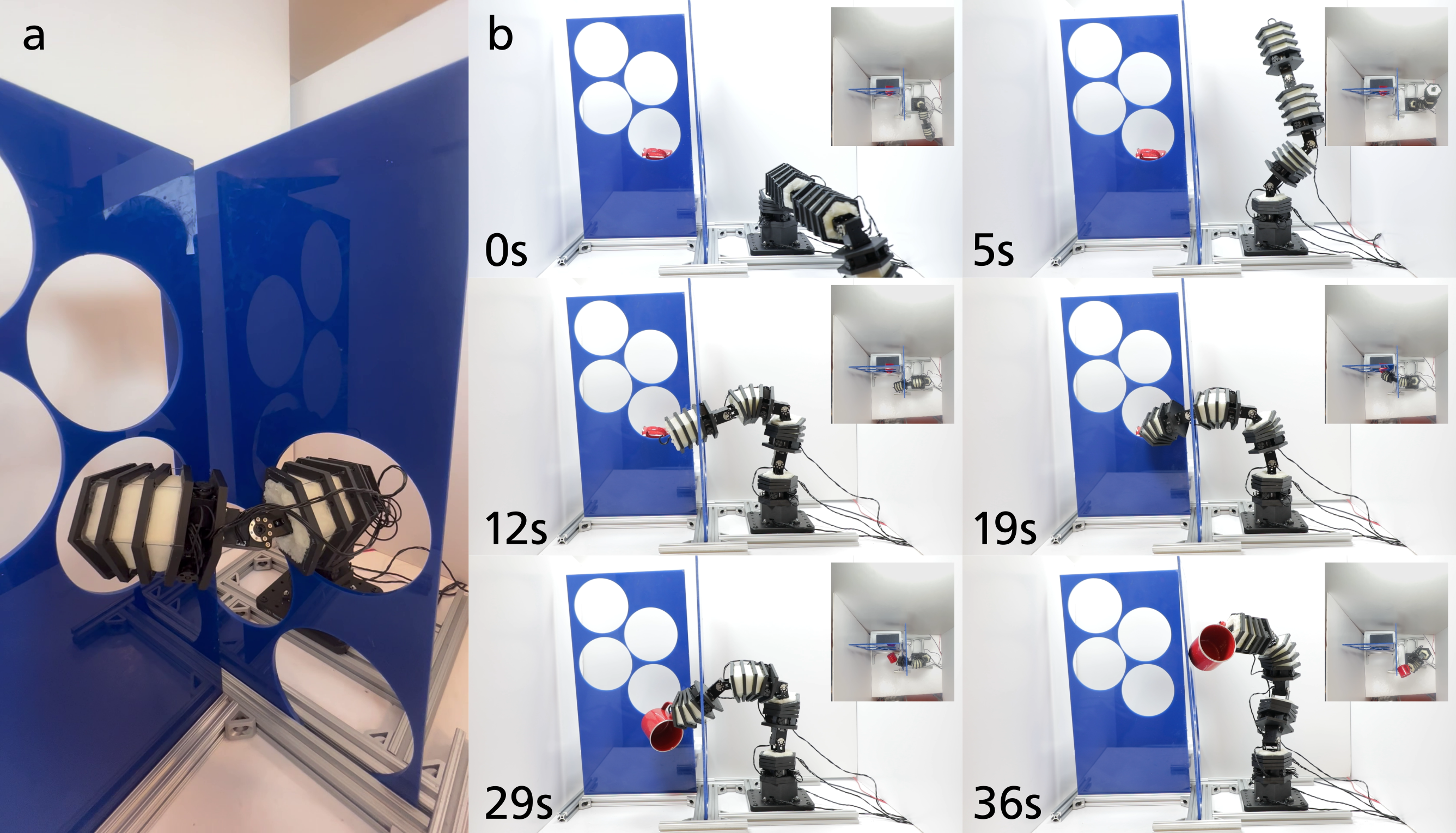}
\caption{\textbf{Experimental Results of Red mug Task:} (a) Demonstration of the capability of the soft-rigid manipulator to traverse a kinematically difficult obstacle course, utilizing both rigid body and soft body deformation to reach a goal state. (b) Snapshots from a demo of the soft-rigid manipulator retrieving a red mug from the obstacle course. The insets are alternative angles from a different video.}\label{fig:mug}
\end{figure}

\noindent \begin{figure}
\centering
\includegraphics[width=1.0\textwidth]{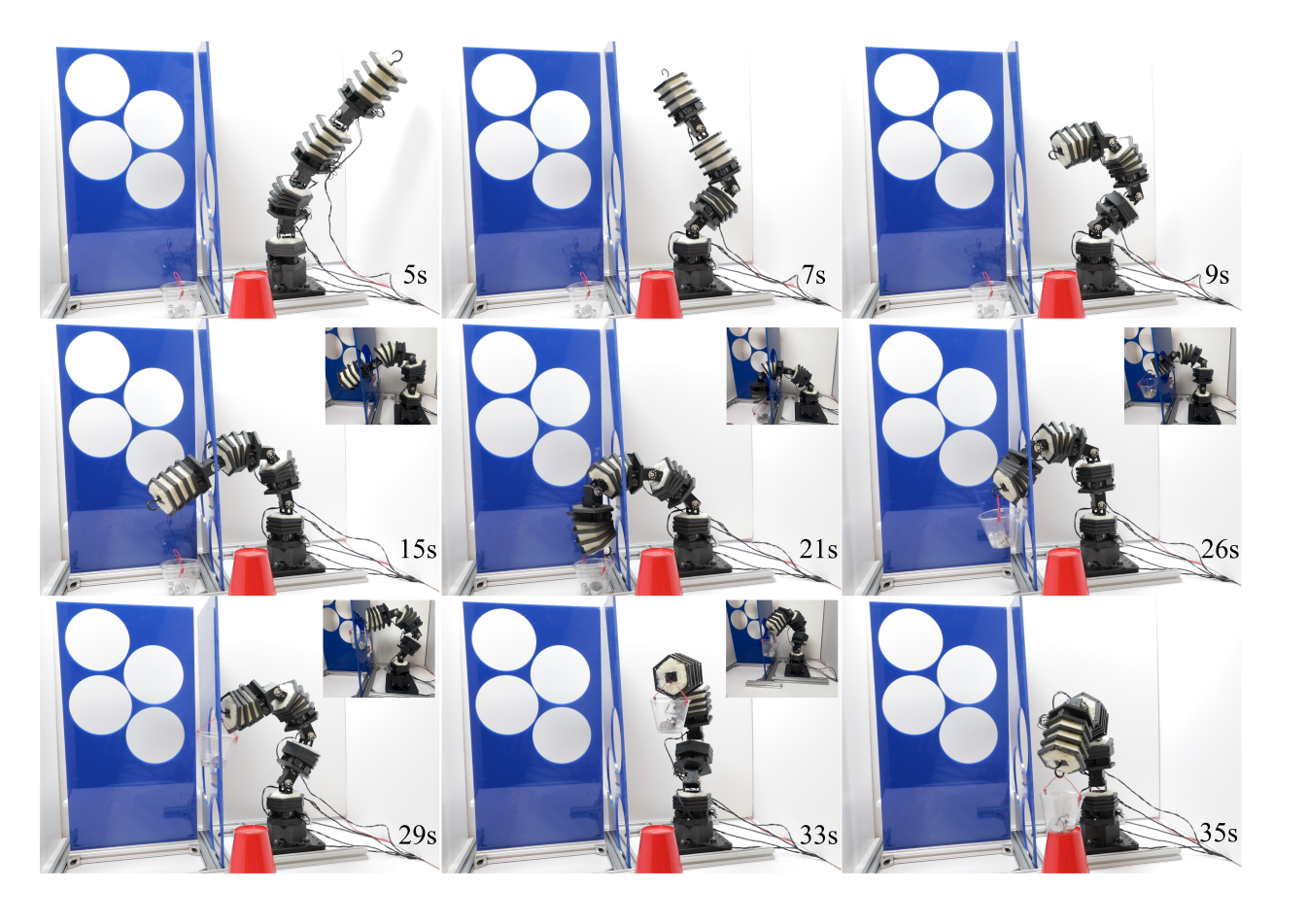}
\caption{\textbf{Experimental Results of Weighted Cup Task:} Features the soft-rigid manipulator going through the obstacle course to pick up a cup with around 200 grams of lab weights. The insets are alternative angles from a different video.}\label{fig:weights}
\end{figure}

\section*{Discussion}
We presented a novel soft-rigid hybrid modular robot arm designed as a series of soft and rigid segments. As we demonstrated (Fig. 
3), each module can change its stiffness, first gradually by contracting and then drastically as the plates make contact. This discrete phenomenon imbues the robot with an approximately "rigid mode" and a variable stiffness "soft mode." We demonstrated that a PD+ controller modified with a term to deal with self contact can dynamically control the manipulator in configuration space on hardware (Fig. 
4). It was possible to exclude the contact term, but this also resulted in the feedback portions of the controller over-stressing the modules and, in a few cases, destroying parts. We also showed an impedance controller for compliant control of the end effector (Fig. 
5). In the SM, we prove boundedness for our new state controller and in our experiments, both controllers are robust to disturbances. Finally, we offer demonstrations of the robot's ability to maneuver objects through an obstacle course, taking advantage of the kinematic redundancy offered by its continuum nature while still being capable of bearing large loads. These tasks would be difficult for both traditional rigid manipulators, because of kinematic difficulties of getting through a complex obstacle course, and soft robots, which would have a hard time precisely manipulating the loads.

The approach as presented has several limitations that are worth discussing. First, our model-based control algorithms utilize the Piecewise Constant Curvature model. While we show that this generally works well, we note that the assumption breaks under large shear loads which can occur when the robot is in "soft mode" and holding a load against gravity. The assumption also breaks under large torsional loads. We made the choice to use PCC because of its relative simplicity and ubiquity in the literature, but to deal with torsion and shear, a less restrictive (but more complex) model such as the Cosserat rod \cite{boyerDynamicsContinuumSoft2021,tillRealtimeDynamicsSoft2019} is required. Utilizing a more complex model would also bring up important control issues. The physical robot cannot directly control shear or torsion, although they are coupled to bending and compression, and thus we would want our controller to deal with the underactuation present in our model. With notable exceptions, underactuated control for soft robots is relatively understudied.

The design of the robot itself also possesses several key limitations. First, the robot's stiffness is coupled to its shape. In particular, as we observed in Figure 
3b, when the robot reaches its stiffest states (fully contracted) it reduces the redundancy of the workspace. This is a feature shared with many stiffness tunable robots in the literature that use similar stiffness modulation strategies \cite{huhDesignAnalysisStiffness2012}. We note that, rather uniquely, our method utilizes the discrete self contact to enable a nonlinear, order of magnitude jump in stiffness. Another limitation of the platform is the design choice to have the three tendon-driving motors physically on the modules, increasing the mass of each module. Placing the motors at the base of the robot as in e.g. \cite{guanTrimmedHelicoidsArchitectured2023} would decrease the weight of each module by approximately 100 grams, thereby significantly increasing the payload of the robot.


\section{Methods}\label{sec11}

\subsection*{Fabrication}
The hexagonal plates, brackets, and rotating base are 3D printed with Onyx material. To cast the foam core of each module, we use a 3D printed mold that ensures a consistent 25mm displacement between each plate. We then pour 20mL of Smooth-On Flex Foam-iT! III solution previously mixed with a 2:1 ratio into the cavity of the mold and allow it to cure for 2 hours. To ensure the foam only adheres to the struts and inner parts of the plates, we brush protected parts of the plates with Ease Release 2831.    
\subsection*{Design}
The manipulator is comprised of four soft-rigid modules. Each module is made up of a several 3D printed plates made from Markforged Onyx® carbon fiber filled nylon and a deformable foam core made from Smooth-On Flex Foam-iT! III Flexible Polyurethane Foam. The foam core has a volume of 307 $cm^3$. The hexagon-shaped plate has struts that protrude into the foam to provide more surface area for the foam to adhere to (refer to Fig 
2c). To compress the plates, we use three Dynamixel XC330-M288-T motors which each have a 3D printed spool connected to the horn of the motor. As shown in Fig. 
2a, the modules rely on cable-driven actuation. When the cables are loose, the module is soft and compliant. As the motors tighten the cables they pull the hexagon plates closer together, compressing the foam. This compression eventually leads the hexagon plates to make complete contact, making the module essentially rigid. To connect each module, we have an XM430-W350-R between each module, with the bottom of the motor screwed to the bottom of one module, and the top of the XM430 screwed onto a FR12-H101K bracket mounted onto the top of the second module. The bottom-most module of the arm is mounted onto a 3D-printed rotating platform actuated with an XM430-W350-R motor to allow for the robot to rotate about its z-axis. The length of the entire manipulator, including the rotating base, stands at 0.68m. 
\subsection*{Electronics and Code}
In order to control both the module's motors and the joint motors in between, we use two U2D2 Dynamixel Power Hub Boards, along with the corresponding USB serial communication converter (U2D2) that connects to the laptop via a micro-USB cable. The XC330s are daisy chained together using 3-pin JST cables-- the XM430s are daisy chained the same way. While both the XC330 and XM430s run on 12V and theoretically could be controlled using a single Power Hub, for simplicity of cabling we chose to designate a power hub for each motor type. If stacking more than three modules for the arm, a 3P JST Expansion board from Dynamixel can also be used to accommodate the additional motors. It is important to note that only around 6 XC330s can be daisy-chained at a time, so we recommend daisy-chaining a maximum of two modules together to a port on the power hub.

We provide an open-source repository that uses Dynamixel's SDK API in Python to control all motors of the soft-rigid arm via packet communication. We have verified this repository works for both Windows and Linux. 

To calculate the state of our manipulator in software, we must keep track of all cable lengths. We accomplish this by relying on the encoders embedded in the motors to measure changes in angle. Both the XC330s and XM430s have an embedded absolute encoder with a 12-bit resolution from AMS. During calibration, we contract the cables to the point where the cables slightly start initiating module contraction. We then measure the distance between each of the motor's spools to the top of the module and refer to this as the neutral or home cable lengths. Using the Dynamixel SDK API, we then call a reading of all the motor's angle measurements and store them as reference angles. Each reference angle can then be used to keep track of changes in angle as the module contracts. We rely on the change in angle measurement to calculate the cable lengths as $l_{i} = l_{0} - (\Delta \theta_{i} r)$, where $l_i$ is the current length of cable $i$, $l_0$ is the original cable length, $\delta \theta$ is the change in motor angle, and $r$ is the radius of the spool.

In order to be used in a control loop, the packet communication between the computer and motors must run at a fast enough rate. We accomplish this by using Dynamixel's Protocol 2.0 option, setting the baud-rate to at least 2 Mbps, as well as setting the USB latency timer on the computer's ports to 1 ms. This allows our control loop to run at 100-200 Hz. It is also required to use the Bulk Read and Bulk Write methods from Dynamixel's SDK such that we can simultaneously read and send packets to both the module and joint motors. 
\subsection*{Characterization}
As seen in Fig. 
3, we demonstrate the relationship between stiffness and cable contraction. We conduct each trial by setting all three cables at the same set length and hanging a lab weight by the end of the module. This was done for five  cable lengths and three weights. We calculate the effective stiffness of the module by measuring the displacement $h$ (see Fig. 
3) at the end of the module for each trial. Between trials we allowed the module to rest for 3 minutes to allow it to return to its complete decompressed state, reducing the effects of memory on the foam.

\subsection*{Obstacle Course Experiments}

Both experiments use the same two plates of 30.5x61 mm acrylic held up by several bars of 8020 aluminum. Each plate was laser cut with four holes of 152mm in diameter. The two acrylic plates are aligned at a 90 degree angle from each other. We also attach a 31.5mm long hook onto the end of the fourth soft-rigid module of the arm to allow the manipulator to pick up both types of cups in the experiments.
For the red mug experiment, we placed a metal red mug taken from the YCB dataset \cite{calli2015benchmarking} behind the bottom right circular opening of the left-most acrylic plate, as seen in Fig 6.
We align the mug such that its handle is protruding from the opening.
For the weighted cup experiment, we place four lab weights, in total weighing 200g, into a plastic cup with an attached handle.

For both experiments we define a series of states with corresponding time points and generate a more comprehensive trajectory using MATLAB's cubicpolytraj function. This gives us a list of desired states, velocities and accelerations that we pass into our controller. For both trajectories we set the first state of all four modules to be completely soft ($dL = 0$) and the motor joints such that the manipulator is resting on the tabletop at the beginning of each experiment. 

\backmatter

\section{Supplementary Information}

Supplementary Text\\
Videos 1-6

\section{Acknowledgements}

This work was done with the support of National Science Foundation EFRI program under grant number 1830901 and the Gwangju Institute of Science and Technology.

\section*{Declarations}

\subsection*{Funding} 
This work was done with the support of National Science Foundation EFRI program under grant number 1830901 and the Gwangju Institute of Science and Technology.
\subsection*{Data and Code Availability} The datasets generated and/or analysed during the current study are available in the Github repository, https://github.com/zpatty/spongebot. The underlying code for this study is available in Github and can be accessed via this link https://github.com/zpatty/spongebot. 
\subsection*{Contributions} ZJP and ES contributed to conception, design, fabrication, code, algorithms, experiments, and writing. CDS and DR contributed to conception and writing.
\subsection*{Competing Interests}
The authors declare no competing interests.


\bibliography{sn-bibliography}
\pagebreak
\begin{center}
\textbf{\large Supplemental Materials: Design and Control of Modular Soft-Rigid Hybrid Manipulators with Self-Contact}
\end{center}
\setcounter{equation}{0}
\setcounter{figure}{0}
\setcounter{table}{0}
\setcounter{page}{1}
\makeatletter
\renewcommand{\theequation}{S\arabic{equation}}
\renewcommand{\thefigure}{S\arabic{figure}}
\renewcommand{\bibnumfmt}[1]{[S#1]}
\renewcommand{\citenumfont}[1]{S#1}

\section{Stability proof for state controller}\label{sec1}

For the following section, we repeat the definitions of our equations from the manuscript for completeness. To describe our state, we utilize the Piecewise Constant Curvature (PCC) kinematic model discussed in \cite{della_santina_improved_2020}, which is singularity-free. Note that the controllers could also be directly extended to systems undergoing torsion or shear deformation by employing a piecewise functional strain approach \cite{renda2020geometric}. However, the PCC approximation was sufficient for modeling the proposed robot, so we focused on it for the sake of simplicity of derivations.

The state for each module is ${q}_{\mathrm{mod},i} = [\Delta_{x,i}, \Delta_{y,i}, \delta L_{i}]^T$. Each rigid joint performs a pure rotation $\theta_i$ \cite{murray_mathematical_1994}. The total state is then ${q} = [\theta_1, {q}_{\mathrm{mod},1}, ..., \theta_n,{q}_{\mathrm{mod},n}]^T$, where $n$ is the total number of modules. The dynamics of the robot are
\begin{equation}\label{eq:dynamics}
    M({q}){\ddot q} + C({q},{\dot q}){\dot q} + G({q}) + K({q}) + D{\dot q} = A({q}){\tau} + J^T f_{\mathrm{ext}},
\end{equation}
where $M({q})$ is the inertia matrix, $C({q},{\dot q})$ is the Coriolis matrix, $G({q})$ is the gravitational force, $K({q})$ is the elastic force, $D$ is the damping matrix, $A({q})$ is the input matrix, ${\tau}$ is the input vector, $J$ is the end effector Jacobian, and $f_{\mathrm{ext}}$ is a force on the end effector, which we assume to be zero in the state controller setting. A key assumption in this work is that the discretization level is such that the number of degrees of freedom (size of $q$) is the same as that of the degrees of actuation (size of $\tau$).

We proposed in the main manuscript the following controller:
\begin{equation}\label{eq:state_c}
    {\tau} = A^{-1}(M{\ddot q_d} + C{\dot q_d} + G({q}) + K{q_d} + D{\dot q_d} + K_{\mathrm P}({q_d} - {q}) + K_{\mathrm D}({\dot q_d} - {\dot q}) + {F}_c),
\end{equation}
where ${q_d}(t)$ is the desired trajectory, $K_{\mathrm P}$ is the proportional feedback gain matrix, $K_{\mathrm D}$ is the derivative feedback gain matrix, and ${F}_c$ is a contact compensation term defined as
\begin{equation}
    {F}_{c,i} = \sigma(c_i({q}))(K{q_d} + K_{\mathrm P}{q_d} - K_{\mathrm P}{S}({q})) = \sigma(c_i({q}))(\hat K{q_d} - K_{\mathrm P}{S}({q})),
\end{equation}
where $c_i({q})$ is the function for the distance between the plates derived from forward kinematics and $\sigma(c_i({q}))$ is a sigmoid function that saturates as $c_i \rightarrow 0$:
\begin{equation}
    c_i = 2(d\frac{L_0 + \delta L_i}{\sqrt{\Delta_{x,i}^2 + \Delta_{y,i}^2}} - d)\sin(\frac{\sqrt{\Delta_{x,i}^2 + \Delta_{y,i}^2}}{6d}),
\end{equation}
\begin{equation}
    \sigma(c) = -\frac{e^{-k_c c}}{e^{-k_c c} + 1}.
\end{equation}
$k_c$ determines the steepness of the sigmoid's slope about $c=0$. 
The term ${S}({q})$ saturates the feedback contribution to ${F}_{c,i}$
\begin{equation}
    {S}({q}) = [\phi(\Delta_{x,i}), \phi(\Delta_{y,i}), \phi(\delta L_{i})]^T,
\end{equation}
where
\begin{equation}
    \phi(\Delta_{x,i}) = 2\Delta_{x,\mathrm{max}} \frac{1-e^{-\Delta_{x,i}/\Delta_{x,\mathrm{max}}}}{1+e^{-\Delta_{x,i}/\Delta_{x,\mathrm{max}}}}
\end{equation}
and $\Delta_{x,\mathrm{max}}$ is a positive constant. ${F}_c$ can then be constructed as 
\begin{equation}
    {F}_c = [0, {F}_{c,1}, ..., 0, {F}_{c,n}]^T,
\end{equation}
where the null terms correspond to the motor states $\theta_i$. Given that the functions $0<\sigma(c({q}))<1$ and $|\phi(x)|<\alpha_s$, where $\alpha_s$ is some positive constant, are bounded by definition, ${F}_c$ is a linear combination of bounded functions and is thus bounded by a constant $\alpha_{{F}_c}$
\begin{equation}
    {F}_{c,i} = \sigma(c_i({q}))(\hat K{q_d} - K_{\mathrm P}{S}({q})) \leq \lambda_{\mathrm{max}}(\hat K)||{q_d}||_{\mathrm M} + \lambda_{\mathrm{max}}(K_{\mathrm P})\alpha_s \leq \alpha_{{F}_c},
\end{equation}
where ${q_d}||_{\mathrm M}$ is the maximum norm of our desired trajectory. 


We can write the closed loop system of (\ref{eq:dynamics}) and (\ref{eq:state_c}) as
\begin{equation}\label{eq:cl}
    M({q}){\ddot e} + C({q},{\dot q}){\dot e} + \hat D{\dot e} + \hat K {e} + {F}_c = 0,
\end{equation}
where $\hat D = D + K_{\mathrm D}$, $\hat K = K + K_{\mathrm P}$, ${e} = {q} - {q_d}$. By design, the controller (\ref{eq:state_c}) will not always (and may never) have as its equilibrium ${e}=0$. However, we can derive conditions for which the closed loop system (\ref{eq:cl}) is uniformly globally bounded and uniformly globally ultimately bounded.

\begin{theorem}\label{thm1}
There exists a $K_{\mathrm{P}}$ and $K_{\mathrm{D}}$ such that the trajectories of the closed loop system (\ref{eq:cl}) are uniformly globally bounded and converge to the set $B_b = \{ e \in {R}^n \: : \: || e(t)|| \leq b\}$.
\end{theorem}

\begin{proof}
Consider the Lyapunov candidate 
\begin{equation}
    V({e},{\dot e}) = \frac{1}{2}{\dot e}^T M({q}) {\dot e} + \frac{1}{2}{e}^T \hat K {e}.
\end{equation}
It is immediate to check that this is a proper candidate - i.e., positive definite in the equilibrium $(0,0)$ - as $V$ is the summation of two quadratic forms with $M$ and $\hat{K}$ positive definite matrices.

From Khalil \cite{khalil2002control}, to show uniform boundedness, we must show that there is some $W_3(x)$, $\mu$ such that $\dot V(t,x) \leq -W_3(x), \; \forall \, ||x|| \geq \mu > 0$. Thus
\begin{equation}
\begin{split}
    \dot V({e},{\dot e}) &= \frac{1}{2}{\dot e}^T \dot M {\dot e} + {\dot e}^T M {\ddot e} + {\dot e}^T \hat K {e}\\
    &= \frac{1}{2}{\dot e}^T \dot M {\dot e} - {\dot e}^T (C{\dot e} + \hat D{\dot e} + \hat K {e} + {F}_c) + {\dot e}^T \hat K {e}\\
    &= - {\dot e}^T \hat D{\dot e} - {\dot e}^T {F}_c,
\end{split}
\end{equation}
where we substituted the closed loop dynamics (\ref{eq:cl}) and used the property ${\dot e}^T \dot M {\dot e} - 2{\dot e}^T C{\dot e} = 0$ \cite{murray_mathematical_1994}. From this we can readily see that 
\begin{equation}
    \dot V \leq - ||{\dot e}||^2 \lambda_{\mathrm{min}}(\hat D), \quad \forall \: ||{\dot e}|| \geq \frac{\alpha_{{F}_c}}{\lambda_{\mathrm{min}}(\hat D)} = \mu.
\end{equation}
As this holds globally, we can conclude that closed loop system (\ref{eq:cl}) is globally uniformly bounded. To derive the ultimate bound, we need to find class $\mathcal K$ functions $\alpha_1(|| y||)$, $\alpha_2(|| y||)$ such that $ \alpha_1(|| y||) \leq V(t, y) \leq \alpha_2(|| y||)$. Taking $ y = [{e}, {\dot e}]^T$ and using the fact from \cite{della_santina_model-based_2020} that $\gamma_3 + \gamma_1 || e|| + \gamma_2 || e||^2 \geq \gamma_0 + \gamma_1 || q|| + \gamma_2 || q||^2 \geq ||M|| \geq 0$, we get
\begin{equation}
\begin{split}
    \alpha_1(|| y||) &= \frac{1}{2}(\lambda_{\mathrm{min}}(\hat K) + \sigma_{\mathrm{min}}(M))|| y||^2 \leq V\\
    \alpha_2(|| y||) &= \frac{1}{2} (\lambda_{\mathrm{max}}(\hat K)|| y||^2 + \gamma_3 || y||^2 + \gamma_1 || y||^3 + \gamma_2 || y||^4) \geq V.
\end{split}
\end{equation}
We can then find ultimate bound $b$, proving the thesis:
\begin{equation}
    b = \alpha_1^{-1}(\alpha_2(\mu)) = \sqrt{\frac{(\lambda_{\mathrm{max}}(\hat K) + \gamma_3)\mu^2 + \gamma_1 \mu^3 + \gamma_2 \mu^4}{\lambda_{\mathrm{min}}(\hat K) + \sigma_{\mathrm{min}}(M)}}.
\end{equation}
\end{proof}




\end{document}